\documentclass[11pt]{article}

\usepackage[preprint]{acl}

\usepackage{times}
\usepackage{latexsym}

\usepackage[T1]{fontenc}

\usepackage[utf8]{inputenc}

\usepackage{microtype}

\usepackage{inconsolata}

\usepackage{graphicx}

\usepackage{hyperref}       
\usepackage{url}            
\usepackage{nicefrac}       
\usepackage{float}

\usepackage[dvipsnames]{xcolor}
\usepackage[most]{tcolorbox}

\usepackage{caption}
\usepackage{amsmath}
\usepackage{bm}
\usepackage{amssymb}
\usepackage{amsthm}
\usepackage{booktabs}
\usepackage{multirow}

\usepackage[ruled,noend,linesnumbered]{algorithm2e}
\DontPrintSemicolon
\makeatletter
\renewcommand{\algocf@Vsline}[1]{
  \strut\par\nointerlineskip
  \algocf@bblockcode%
  \algocf@push{\skiprule}
  \hbox{\textcolor{gray!50!white}{\vrule}
    \vtop{\algocf@push{\skiptext}
      \vtop{\algocf@addskiptotal #1}}}
  \algocf@pop{\skiprule}
  \algocf@eblockcode%
}
\makeatother

\SetCommentSty{mycommfont}

\makeatletter
\newcommand{\removelatexerror}{\let\@latex@error\@gobble}
\makeatother

\usepackage{colortbl}
\newcolumntype{a}{>{\columncolor{Gray!20}}c}

\usepackage{tikz}
\usepackage{pgfplots}
\pgfplotsset{compat=1.18}
\usetikzlibrary{math,calc,positioning,intersections,fit,matrix,datavisualization,patterns.meta,decorations.pathreplacing,decorations.markings,arrows.meta,shapes.geometric,pgfplots.groupplots,shapes.misc,patterns}
\usepgfplotslibrary{fillbetween}
\tikzset{>=latex}

\pgfdeclarelayer{background}
\pgfsetlayers{background,main}

\usepackage{colortbl}
\newcolumntype{a}{>{\columncolor{Gray!20}}c}
\newcolumntype{i}{>{\columncolor{Gray!10}}c}
\newcolumntype{j}{>{\columncolor{Gray!20}}c}

\def\eqref#1{equation~\ref{#1}}

\def\1{\bm{1}}

\def\va{{\bm{a}}}

\def\vr{{\bm{r}}}

\def\vu{{\bm{u}}}
\def\vv{{\bm{v}}}
\def\vw{{\bm{w}}}
\def\vx{{\bm{x}}}
\def\vy{{\bm{y}}}

\def\eva{{a}}

\def\evw{{w}}

\def\evy{{y}}

\DeclareMathAlphabet{\mathsfit}{\encodingdefault}{\sfdefault}{m}{sl}
\SetMathAlphabet{\mathsfit}{bold}{\encodingdefault}{\sfdefault}{bx}{n}

\DeclareMathOperator*{\argmax}{arg\,max}

\usepackage{stmaryrd}

\newtheorem{definition}{Definition}
\newtheorem{proposition}{Proposition}

\title{Suffix-Constrained Greedy Search Algorithms for Causal Language Models}

\author{
    Ayoub Hammal\textsuperscript{1} \qquad Pierre Zweigenbaum\textsuperscript{1} \qquad Caio Corro\textsuperscript{2} 
    \\
    \textsuperscript{1}Université Paris-Saclay, CNRS, LISN
    \\
    \textsuperscript{2}INSA Rennes, IRISA, CNRS, Université de Rennes
    \\
    \texttt{\{ayoub.hammal,pz\}@lisn.fr} \qquad\texttt{caio.corro@irisa.fr}
}

\begin{document}
\maketitle
\begin{abstract}
Large language models (LLMs) are powerful tools that have found applications beyond human-machine interfaces and chatbots.
Beside free-form generation,
there has been an interest in constrained generation,
a setting where LLMs are constrained to generate well-formed outputs with respect to the language defined by a formal grammar.
Although appealing, this setting may be over restrictive for downstream applications.
For example, many LLM tasks require the model to reason freely before generating its final response in a specific format.

In this work, we introduce suffix-constrained generation, a constrained generation setting in which only the end of the response is constrained by a grammar, a scenario that is not supported by existing constrained generation methods.
We introduce several suffix-constrained generation algorithms that are based on greedy search.
We experiment on several datasets, and show that our approach allows to guarantee suffix constraints without having a negative impact on results, and even improving them in many settings.
\end{abstract}

\section{Introduction} \label{sec:introduction}

Large language models (LLMs) exhibit strong generative and reasoning capabilities and are used to tackle a wide range of downstream tasks.
In many practical settings,
there is a need to control LLM outputs, for example in order to follow an expected format or to restrict vocabulary to a predefined set.
Although one can rely on fine-tuning, careful prompting or in-context learning to achieve output controllability \citep{ouyang2022followinstructions, chen2023promptingorfinetuning, cruz2025prompt},
off-the-shelf generation methods still produce free-form text and none of these technics can formally guarantee their well-formedness \citep{li2024measuring, tang2023context}.
As such, there is an interest in constrained generation.
In this setting, one defines the (possibly infinite) language (or set) of well-formed outputs, for example using a finite-state automaton or a context-free grammar, and then, at each generation step, only allows for tokens that lead to well-formed outputs \citep{outlines,geng2023grammarconstraineddecodingforstructured,ugare2024syncodellmgenerationgrammar}.
However, restricting the full output to be in the language of a predefined grammar may be overly restrictive.
For example, free-form reasoning steps improve downstream performance \citep{wei2022chainofthought,tam2024speakfreely}.
As such, there is a need for a middle ground between free-form and constrained generation \cite{banerjee2025crane}.
We introduce the \emph{suffix-constrained generation} problem:
restrict an LLM to produce only outputs that have at least one suffix (sequence of tokens ending the output) that is in the language of a grammar.

A common application of our problem, and the one on which we evaluate our contributions, is question answering, in which an LLM is expected to generate an output like ``\texttt{<Reasoning> The answer is: <Answer>}'', where \texttt{<Reasoning>} denotes intermediate reasoning tokens, and \texttt{<Answer>} denotes the final answer in a task-specific format.
To this end, the most straightforward approach is to provide formatting instructions in the prompt and to extract the last match of a predefined pattern in the output: the last number, the last element from a finite list (\emph{e.g.}, A, B, C, D), the last expression in a \texttt{\textbackslash boxed\{\}} mathematical \LaTeX\ command, etc.
However, an LLM, especially smaller less capable models, may ignore the formatting instruction during generation or, due to LLM verbosity, produce extra information after the answer, among other complications such as exceeding the generation budget, which hinders automatic extraction of the answer (see examples in App.~\ref{app:motivation}).
Another counterargument one might raise is that there are problems for which describing a grammar in natural language in the prompt may be imprecise or difficult or even drastically increase the number of tokens in the (limited) context window.
As such, using any formal grammar to constrain generation could be easier and more practicable than designing an overly complex prompt.\footnote{For example, the reader can check the formal grammar for generating well-formed \emph{browsable scalable vector graphics}: \url{https://slebok.github.io/zoo/markup/graphical/svg/furubayashi/extracted/index.html}}

In this work, we propose several greedy search algorithms for suffix-constrained generation,
that is generation algorithms that always return an output whose \emph{suffix} belongs to a predefined language of well-formed outputs.
In order to keep the generation cost to a minimum, our algorithms require a beam\footnote{Our algorithms are based on greedy search but may explore two different hypotheses in parallel due to suffix constraints, but they are not beam search algorithms \emph{per se}.} of at most two.
Intuitively, our greedy search algorithms transition from unconstrained token generation steps to automaton/grammar-constrained generation.
As such, in the question answering example, the LLM response is guaranteed to have its final answer extracted correctly using simple pattern-matching techniques.

Our contributions can be summarized as follows:
(1)~we introduce suffix-constrained generation for LLMs;
(2)~we demonstrate that suffix-constrained generation cannot be addressed with the same methods as ``standard'' constrained generation (Proposition~\ref{prop:glutton_constrained});
(3)~we introduce different greedy search algorithms that keep the generation cost to a minimum by maintaining at most two hypotheses in their beam;
(4)~we experiment with OLMo 2 and Gemma 3 on five question-answering datasets, and show that our approach is easy to apply and maintains or even improves downstream task results.

\textbf{Notations.}
The set $V$ is the token vocabulary, and the special end-of-generation symbol is $\bot$, such that $\bot \notin V$.
Moreover, we define $\overline V \triangleq V \cup \{\bot\}$.
We write the next token distribution of an LLM as $p(v | \vx)$ for a given prefix $\vx \in V^\ast$ and token $v\in \overline V$.
We write vectors as $[...]$ and sequences as $\langle ... \rangle$.
We index vectors using elements of any set, for example $\left[ v \right]_{v \in V}$ is a vector containing every elements of the set $V$.
The empty sequence is written $\epsilon$, and $\odot$ is the sequence concatenation operator.
\section{Background}
\label{sec:background}

Given an input $\vx$, generation aims to compute the best output as measured by the causal (or autoregressive) model distribution:
\begin{align}\label{eq:obj}
    \widehat\vy \in \argmax_{\vy \in Y} \prod_{t = 1}^{|\vy|} p(\evy_t | \vx \odot \langle \evy_i\rangle_{i=1}^{t-1}),
\end{align}
where $Y$ is the set of well-formed outputs.
Computing $\widehat\vy$ is intractable in general.
As such, various approximation methods have been proposed, including greedy search and beam search \citep{lowerre1976harpy,graves2012seqtransduction}.

\subsection{Greedy Search}
\label{sec:greedy_generation}

Greedy search (or decoding) is a simple algorithm that approximates the solution to the problem in Equation~\ref{eq:obj} by starting with an empty output and repeatedly concatenating the next token with the highest probability to the current partial output.
In the unconstrained generation scenario, the set of well-formed outputs is defined as:
\[
Y_\text{free} \triangleq V^* \times \{\bot\},
\]
that is,\ $Y_\text{free}$ is the set of strings from the free monoid generated by $V$ followed by the special end-of-generation symbol $\bot$.
The role of the special symbol $\bot$ is twofold.
First, it is necessary for a locally normalized model to induce a well-defined probability distribution over variable-length sequences,
avoiding probability mass leakage to infinite continuations \citep[Theorem~2.5.3]{cotterell2024formalaspects}.
Second, it simplifies inference procedures such as greedy search, since sequence generation can be terminated in a principled manner when $\bot$ is selected as the next token.\footnote{As $\bot \notin V$, when $\bot$ is selected as the next output token, the new partial output is not contained in $V^*$ but is in $Y_\text{free}$.}
The greedy search algorithm for the set of well-formed outputs $Y_\text{free}$ is outlined in Alg.~\ref{alg:greedy_search}.\footnote{Algorithm color code is the following:
\textcolor{brown}{brown is for greedy state;}
\textcolor{violet}{violet is for constrained hypothesis state;}
\textcolor{ForestGreen}{green is for constrained hypothesis interruption condition.}}

\begin{table}[!ht]
\centering
\begin{minipage}[t]{\linewidth}
    \begingroup
\removelatexerror
\begin{algorithm}[H]
    \small
    \caption{\texttt{greedy\_search}}
    \label{alg:greedy_search}
    \KwData{A model $p$. A prefix $\vx$.}
    \KwResult{A response sequence $\vy \in Y_{\text{free}}$.}

    \tcp{Starts with the empty sequence}
    \textcolor{brown}{$\vy \gets \epsilon$} 

    \tcp{The condition reduces to checking if the last token in $\vy$ is $\bot$.}
    \While{\textcolor{brown}{$\vy \notin Y_\text{\upshape free}$}}{
        \textcolor{brown}{$\vw \gets \big[~\log p(v | \vx \odot \vy)~\big]_{v \in \overline V}$} \;
        \textcolor{brown}{$\vy \gets \vy \odot \langle \argmax_{v \in \overline V} \evw_v \rangle$} \; 
    }
    \Return{\textcolor{brown}{$\vy$}} \;
\end{algorithm}
\endgroup
    \begingroup
\removelatexerror
\begin{algorithm}[H]
    \small
    \caption{\texttt{constrained\_gs}}
    \label{alg:constrained_greedy_search}
    \KwData{A model $p$. A prefix $\vx$. A grammar $\mathcal G$.}
    \KwResult{A response sequence $\vy \in Y_\mathcal G$.}

    \textcolor{violet}{$\vy \gets \epsilon$ \;} 

    \While{\textcolor{violet}{$\vy \notin Y_\mathcal G$}}{
        \textcolor{violet}{$\vw \gets \big [~ \log p(v | \vx \odot \vy) ~\big ]_{v \in \overline V}$} \;
        \tcp{Constraint the next token generation using the transition operator}
        \textcolor{violet}{$\vy \gets \vy \odot \langle \argmax_{v \in \overline{\operatorname{next}}_{\mathcal G}(\vy)} \evw_v \rangle$} \;
    }
    \Return{\textcolor{violet}{$\vy$}} \;
\end{algorithm}
\endgroup
    \begingroup
\removelatexerror
\begin{algorithm}[H]
    \small
    \caption{\texttt{greedy\_pipeline\_search}}
    \label{alg:pipeline_search}
    \KwData{A model $p$. A prefix $\vx$. A grammar $\mathcal G$.}
    \KwResult{A response sequence $\vy \in Y_{\to\mathcal G}$.}

    \textcolor{brown}{
        $\vy \gets \texttt{greedy\_search}(p, \vx)
    $}\;
    
    \textcolor{violet}{
        $\vx' \gets \vx \odot \langle y_i \rangle_{i=1}^{|\vy|-1}$
    } \;

    \textcolor{violet}{
        $\vy' \gets \texttt{constrained\_gs}(p, \vx', \mathcal G)$
    }\;

    \Return{\textcolor{violet}{$\langle y_i \rangle_{i=1}^{|\vy|-1} \odot \vy'$}} \;
\end{algorithm}
\endgroup
\end{minipage}
\end{table}

\subsection{Constrained Generation}
\label{sec:constrained_gen}

In the constrained generation scenario, we are given a grammar $\mathcal G$, and we aim for outputs from the language generated by $\mathcal G$, denoted $\operatorname{lang}(\mathcal G) \subseteq V^\ast$.
The set of well-formed outputs is then defined as:
\[
Y_{\mathcal G} \triangleq \operatorname{lang}(\mathcal G) \times \{ \bot \},
\]
that is, it is the set of strings in the language generated by $\mathcal G$ followed by the end-of-generation token.
To constrain the generation of sequences of $Y_{\mathcal G}$ only, one needs to change the greedy search algorithm so that each generation step is guaranteed to lead to a sequence in $Y_{\mathcal G}$ \citep{outlines,geng2023grammarconstraineddecodingforstructured,ugare2024syncodellmgenerationgrammar}.

\begin{definition}[Prefix] \label{def:prefix_language}
    Given a grammar $\mathcal G$ defined over a vocabulary $V$,
    the prefix language of $\mathcal G$ is noted as $\operatorname{pref}(\mathcal G)$ and is defined as:
    \[
        \operatorname{pref}(\mathcal G) \triangleq \{ \vu \in V^* | \exists \vv \in V^*: \vu \odot \vv \in \operatorname{lang}(\mathcal G) \}.
    \]
\end{definition}

\begin{definition}[Transition operator] \label{def:transition_operator}
Given grammar $\mathcal G$ and a sequence $\vy \in \operatorname{pref}(\mathcal G)$,
we define the transition operator $\operatorname{next}_{\mathcal G}$ as the set of tokens
that can be concatenated to a given sequence to build a new sequence that is in the prefix language of $\mathcal G$:
\[
    \operatorname{next}_{\mathcal G}(\vy)
    \triangleq
    \left\{
        v \in V
        |
        \vy \odot \langle v \rangle \in \operatorname{pref}(\mathcal G)
    \right\}.
\]
For convenience, we also define an extended operator that includes the end-of-generation token:
\[
    \overline{\operatorname{next}}_{\mathcal G}(\vy)
    \triangleq
    \operatorname{next}_{\mathcal G}(\vy)
    \cup
    \begin{cases}
        \{ \bot \} &\text{if~} \vy \in \operatorname{lang}(\mathcal G),\\
        \emptyset &\text{otherwise}.
    \end{cases}
\]
\end{definition}

With these definitions, we can modify the standard greedy search algorithm to constrain the generation to the language of $\mathcal G$, as shown in Alg.~\ref{alg:constrained_greedy_search}.

To compute $\operatorname{next}_{\mathcal G}(\vy)$, we can use any left-to-right parser available for $\mathcal G$.
For example, if $\mathcal G$ is a context-free grammar (CFG),
the set of allowed next tokens for the next position can be computed using the Earley algorithm \cite{earley,stolcke1995,opedal2023efficient}.

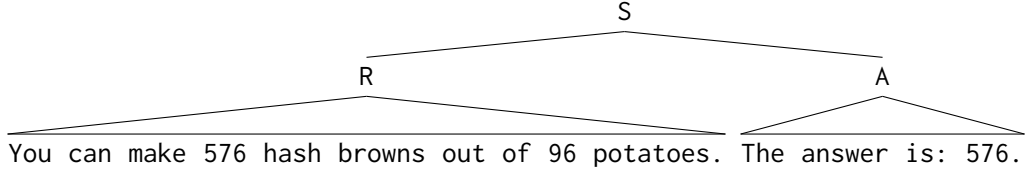
\begin{figure*}
\center
\begin{tikzpicture}[
    every node/.style={
        rectangle,
        inner xsep=0cm,
        inner ysep=0.1cm,
        text height=1.5ex,
        text depth=.25ex,
    }
]
            
    \node (w1) [rectangle] {\texttt{You}};
    \node (w2) [rectangle, right=0.2cm of w1] {\texttt{can}};
    \node (w3) [rectangle, right=0.2cm of w2] {\texttt{make}};
    \node (w4) [rectangle, right=0.2cm of w3] {\texttt{576}};
    \node (w5) [rectangle, right=0.2cm of w4] {\texttt{hash}};
    \node (w6) [rectangle, right=0.2cm of w5] {\texttt{browns}};
    \node (w7) [rectangle, right=0.2cm of w6] {\texttt{out}};
    \node (w8) [rectangle, right=0.2cm of w7] {\texttt{of}};
    \node (w9) [rectangle, right=0.2cm of w8] {\texttt{96}};
    \node (w10) [rectangle, right=0.2cm of w9] {\texttt{potatoes.}};

    \node (w11) [rectangle, right=0.2cm of w10] {\texttt{The}};
    \node (w12) [rectangle, right=0.2cm of w11] {\texttt{answer}};
    \node (w13) [rectangle, right=0.2cm of w12] {\texttt{is:}};
    \node (w14) [rectangle, right=0.2cm of w13] {\texttt{576.}};

    \draw (w1.north west) --coordinate (part1)  (w10.north east);
    \coordinate[above=0.2cm of part1] (part1_b);
    \draw (w1.north west) -- (part1_b);
    \draw (w10.north east) -- (part1_b);

    \draw (w11.north west) --coordinate (part2)  (w14.north east);
    \coordinate[above=0.2cm of part2] (part2_b);
    \draw (w11.north west) -- (part2_b);
    \draw (w14.north east) -- (part2_b);

    \node (R) [rectangle,anchor=south] at (part1_b) {\texttt{R}};
    \node (A) [rectangle,anchor=south] at (part2_b) {\texttt{A}};

    \coordinate (RA) at ($(R)!0.5!(A)$);
    \node (S) [rectangle,anchor=south, above=0.5cm of RA] {\texttt{S}};

    \draw (S.south) -- (R.north);
    \draw (S.south) -- (A.north);
    
\end{tikzpicture}

\caption{Exemple of a derivation with a suffix-constrained grammar ${\to}\mathcal G$. The \texttt{R} non-terminal derives to the reasoning sequence whereas \texttt{A} derives to the constrained answer.}
\label{fig:derivation}

\end{figure*}

\section{Suffix-Constrained Generation} \label{sec:suffix_constrained_generation}

In the suffix-constrained generation scenario,
we assume we have access to a grammar $\mathcal G$ that defines the expected format of the answer, \emph{e.g.}, ``\texttt{The answer is: <Answer>.}'',
where \texttt{<Answer>} may be constrained to follow any expected format (\emph{e.g.}, a floating point number).
Our goal is to allow the LLM to freely generate reasoning steps in the beginning of the output, but constrain the end of the output to follow this grammar.
In this section,
we show that we cannot rely on the constrained generation approach described in Sec.~\ref{sec:constrained_gen} for this setting, and propose a first simple solution to this problem.

\subsection{Issue with Suffix-Constrained Grammars}

We now formalize the notion of suffix-constrained grammar.

\begin{definition}[Suffix-constrained grammar]
    Let $\mathcal G$ be a grammar on vocabulary $V$.
    A suffix-constrained grammar ${\to}\mathcal G$ of $\mathcal G$ is a grammar whose language is defined as:
    \[
    \operatorname{lang}({\to}\mathcal G)
    =
    V^* \times \operatorname{lang}(\mathcal G).
    \]
\end{definition}

Without loss of generality, we assume $\mathcal G$ is a CFG with a single start symbol \texttt{A}.
As $V^*$ is trivially a context-free language,
we can build ${\to}\mathcal G$ as a CFG including all production rules of $\mathcal G$,
and extra non-terminals \texttt{S} and \texttt{R}, and the following extra production rules:
\[
    \texttt{S} \to \texttt{R}~\texttt{A} \qquad \texttt{R} \to \epsilon
    \qquad
    \texttt{R} \to v \texttt{R},~\forall v \in V
\]
Setting \texttt{S} as the start symbol, we obtain the expected grammar, see the derivation in Figure~\ref{fig:derivation} as an example.\footnote{Note that $V^*$ is a regular language. As such, a similar construction can also be realized with finite-state automata.}
In our setting, the set of well-formed outputs for the LLM is defined as:
\[
    Y_{{\to}\mathcal G}
    \triangleq \operatorname{lang}({\to}\mathcal G) \times \{\bot\}
    = V^* \times \operatorname{lang}(\mathcal G) \times \{\bot\}.
\]
It is therefore tempting to consider constrained greedy search on the grammar ${\to}\mathcal G$ as a method for suffix-constrained generation.
The next proposition shows that such an approach is impractical.

\begin{proposition}\label{prop:glutton_constrained}
Let $\mathcal G$ be a grammar on vocabulary $V$.
For any prefix $\vy \in \operatorname{pref}({\to}\mathcal G)$ of the suffix-constrained grammar, we have $\operatorname{next}_{{\to}\mathcal G}(\vy) = V$.
\end{proposition}
\begin{proof}
The proof is trivial as for all $v \in V$ we have \(\vy \cdot \langle v \rangle \in V^*\). As such, non-terminal $\texttt{R}$ in our construction can be derived into $\vy \odot \langle v \rangle$,
and it is therefore in the prefix language of ${\to}\mathcal G$.
\end{proof}

Proposition~\ref{prop:glutton_constrained} shows that using the constrained greedy search method does not work, as the extended transition operator is simply:
\[
    \overline{\operatorname{next}}_{{\to}\mathcal G}(\vy) = V
    \cup
    \begin{cases}
        \{ \bot \} &\text{if~} \vy \in \operatorname{lang}({\to}\mathcal G),\\
        \emptyset &\text{otherwise}.
    \end{cases}
\]
In other words, the grammar is not used to explicitly guide the LLM to output an answer in the expected format:
the model will (maybe) emit an end-of-generation token if and only if unconstrained greedy search also generates a sequence in $\operatorname{lang}({\to}\mathcal G)$.
In practice, this means that if, for a given prompt, the LLM does not generate an answer in the expected format, it is forced to continue generating without any extra information.
In preliminary experiments, we observed that this always leads to hitting the generation budget.

\subsection{Greedy Pipeline}

We now propose a first simple and naive solution to suffix-constrained generation, which we call \emph{greedy pipeline}.
Given an input prompt $\vx$,
we first apply standard greedy decoding to compute the unconstrained generation $\vy \in Y_\text{free}$.
Then, we build the new prompt $\vx' = \vx \odot \langle y_i \rangle_{i=1}^{|\vy|-1}$ by removing the last $\bot$ token from the generation $\vy$.
Finally, we apply constrained greedy search on $Y_{\mathcal G}$ using $\vx'$ as a prompt.
This process is illustrated in Alg.~\ref{alg:pipeline_search}.
This approach can be easily and efficiently implemented in any generation framework, as long as we know how to search for an output in $Y_{\mathcal G}$.

\section{Beam-Based Algorithms} \label{sec:beam_based_algorithms}

\begin{algorithm}[t!]
    \small
    \caption{\texttt{constrained\_beam}}
    \label{alg:constr_beam}
    \KwData{A model $p$. A prefix $\vx$. A grammar $\mathcal G$.}
    \KwResult{A response sequence $\vy \in Y_{\to\mathcal G}$.}

    \textcolor{brown}{$\vy \gets \epsilon,~s \gets 0$} \; 
    \textcolor{violet}{$\vr \gets \epsilon,~\va \gets \epsilon,~s' \gets -\infty$} \; 

    \While{\textcolor{violet}{$\vr \odot \va \notin Y_{\to\mathcal G}$}}{
        \textcolor{brown}{$\vw \gets \big [~ \log p(v | \vx \odot \vy) ~\big ]_{v \in \overline V}$} \;
        \textcolor{violet}{$\vw' \gets \big [~ \log p(v | \vx \odot \vr \odot \va) ~\big ]_{v \in \overline V}$} \;

        \eIf{\textcolor{ForestGreen}{$s + \max_{v \in \overline{\operatorname{next}}_{\mathcal G}(\epsilon)} \evw_v \newline \null\qquad < s' + \max_{v \in \overline{\operatorname{next}}_{\mathcal G}(\va)} \evw'_v$}}
        {
            \textcolor{violet}{$s' \gets s' + \max_{v \in \overline{\operatorname{next}}_{\mathcal G}(\va)} \evw'_v$} \;
            \textcolor{violet}{$\va \gets \va \odot \langle \argmax_{v \in \overline{\operatorname{next}}_{\mathcal G}(\va)} \evw'_v \rangle$} \;
        }
        {
            \textcolor{violet}{$\vr \gets \vy$} \;
            \textcolor{violet}{$s' \gets s + \max_{v \in \overline{\operatorname{next}}_{\mathcal G}(\epsilon)} \evw_v$} \;
            \textcolor{violet}{$\va \gets \langle \argmax_{v \in \overline{\operatorname{next}}_{\mathcal G}(\epsilon)} \evw_v \rangle$} \;
        }
        
        \textcolor{brown}{$s \gets s + \max_{v \in V} \evw_v$} \;
        \textcolor{brown}{$\vy \gets \vy \odot \langle \argmax_{v \in V} \evw_v \rangle$} \;
    }
    \Return{\textcolor{violet}{$\vr \odot \va$}} \;
\end{algorithm}
\begin{algorithm}[t!]
    \small
    \caption{\texttt{bifurcation\_penalty}}
    \label{alg:constr_derivation}
    \KwData{A model $p$. A prefix $\vx$. A grammar $\mathcal G$. A boolean $\alpha$ indicating if the penalty difference should be computed using logits. A boolean $\beta$ indicating that we must return the last constrained hypothesis instead of the one of minimum bifurcation penalty.}
    \KwResult{A response sequence $\vr \odot \va \in Y_{\to\mathcal G}$.}

    $\vy \gets \epsilon, \vr \gets \epsilon, \va \gets \epsilon$ \tcp*{Hypotheses}
    $\Delta = \infty$ \tcp*{Constr.\ hyp.\ deviation penalty}
    $\Delta_\text{best} \gets \infty, \vr_\text{best} \gets \epsilon, \va_\text{best} \gets \epsilon$
    \tcp*{Best hyp.}
    
    \While{\textcolor{brown}{$\vy \notin Y_\text{\upshape free}$}}{
        \textcolor{brown}{$\vw \gets \big [~ p(v | \vx \odot \vy) ~\big ]_{v \in \overline V}$} \;
        \textcolor{violet}{$\vw' \gets \big [~ p(v | \vx \odot \vr \odot \va) ~\big ]_{v \in \overline V}$} \;

        \tcp{Compute the deviation penalty}
        \eIf{$\alpha$}{
            $\Delta_\text{new}\gets \max_{v \in \overline{V}} \log \evw_v \newline\null\qquad\quad\qquad -\max_{v\in\overline{\operatorname{next}}_{\mathcal G}(\epsilon)} \log \evw_v$ \;
        }{
            $\Delta_\text{new}\gets \max_{v \in \overline{V}} \evw_v \newline\null\qquad\quad\qquad- \max_{v\in\overline{\operatorname{next}}_{\mathcal G}(\epsilon)} \evw_v$ \;
        }

        \eIf{\textcolor{ForestGreen}{$\Delta_\text{new}< \Delta$}}{
            $\Delta \gets \Delta_\text{new}$ \;
            \textcolor{violet}{$\vr \gets \vy$} \;
            \textcolor{violet}{$\va \gets \langle \argmax_{v\in\overline{\operatorname{next}}_{\mathcal G}(\epsilon)} \evw_v \rangle$} \;
        }{
            \textcolor{violet}{$\va \gets \va \odot \langle \argmax_{v\in\overline{\operatorname{next}}_{\mathcal G}(\va)} \evw'_v \rangle$} \;
            \tcp{Update deviation penalty until greedy \& const.\ hyp.\ disagree}
            \If{$\Delta = 0$}{
                $\Delta = \Delta_\text{new}$ \;
            }
            
            \If{\textcolor{violet}{$\va \in Y_{\mathcal G}$}}{
                \If{$\beta \lor \Delta < \Delta_\text{best}$}{
                    $\Delta_\text{best} \gets \Delta$ \;
                    $\vr_\text{best} \gets \vr, \va_\text{best} \gets \va$ \;
                    \If{$\Delta = 0$}{
                        \tcp{Unconstrained output is well-formed}
                        \textbf{break}
                    }
                }
                $\Delta \gets \Delta_\text{new}$ \; 
                \textcolor{violet}{$\vr \gets \vy$} \;
                \textcolor{violet}{$\va \gets \langle \argmax_{v\in\overline{\operatorname{next}}_{\mathcal G}(\epsilon)} w_v \rangle$} \;
            }
        }

        \textcolor{brown}{$\vy \gets \vy \odot \langle \argmax_{v \in \overline V} \evw_v \rangle$} \;
    }

    \If{$\beta \lor \Delta < \Delta_\text{best}$}{
        \While{\textcolor{violet}{$\va \notin Y_\mathcal G$}}{
            \textcolor{violet}{$\vw' \gets \big [~ \log p(v | \vx \odot \vr \odot \va) ~\big ]_{v \in \overline V}$} \;
            \textcolor{violet}{$\va \gets \va \odot \langle \argmax_{v \in \overline{\operatorname{next}}_{\mathcal G}(\va)} \evw'_v \rangle$} \;
        }
        $\vr_\text{best} \gets \vr, \va_\text{best} \gets \va$ \;
    }

    \Return $\vr_\text{best} \odot \va_\text{best}$

\end{algorithm}

In this section, we propose heuristic algorithms based on decoding with a beam of width two.
For all of them,
the main hypothesis in the beam corresponds to unconstrained greedy search (derivation of the non-terminal \texttt{R} in the CFG example),
whereas the secondary hypothesis corresponds to a sequence where the constrained suffix is being generated (derivation of the non-terminal \texttt{A} in the CFG example).
The algorithms differ in how they replace the current secondary hypothesis with a new one and how a completed hypothesis is selected.

\subsection{Constrained Hypothesis Beam Search}
\label{sec:cst_hyp}

We introduce a constrained generation algorithm that is based on a parallel decoding process, using a beam of width two:
the first hypothesis corresponds to unconstrained greedy decoding (\emph{i.e.}, free reasoning),
whereas the second hypothesis corresponds to a generation that transitioned to the constrained answer part.
We will refer to these states as the greedy hypothesis $\vy$ and the constrained answer hypothesis $\vr \odot \va$, respectively, where $\vr$ (resp.\ $\va$) corresponds to reasoning (resp.\ constrained answer) tokens.
At each step, we first check if the constrained hypothesis $\vr \odot \va$ should be continued or replaced.
The test is simple: the current constrained hypothesis must be continued if starting the constrained answer after $\vy$ leads to a hypothesis of lower score:
\begin{align*}
    &\overbrace{s + \max\nolimits_{v \in \overline{\operatorname{next}}_{\mathcal G}(\epsilon)} \log p(v | \vx \odot \vy)}^\text{replacement score} <
    \\
    & \qquad \underbrace{s' + \max\nolimits_{v \in \overline{\operatorname{next}}_{\mathcal G}(\va)} \log p(v | \vx \odot \vr \odot \va)}_\text{continuation score},
\end{align*}
where $s$ (resp.\ $s'$) is the log probability of the greedy (resp.\ constrained) hypothesis, that is, $\log p(\vy | \vx)$ (resp.\ $\log p(\vr \odot \va | \vx)$).
Replacing the constrained hypothesis simply reduces to setting the reasoning tokens $\vr$ to the current greedy hypothesis tokens $\vy$,
and $\va$ to the first token of the answer, an illustration is given in App.~\ref{app:constrained_hypothesis_beam}.

Finally, for the greedy hypothesis $\vy$, we simply append the next best token excluding the end-of-generation $\bot$ one, as the greedy hypothesis is not allowed to end.
The generation ends only when the algorithm selects the end-of-generation token $\bot$ as the next token for the constrained hypothesis.
Pseudo-code is given in Alg.~\ref{alg:constr_beam}.

\subsection{Bifurcation Position Penalty}
\label{sec:bifurcation_penalty}

The previous constrained beam algorithm relies on the probability of the partial constrained hypothesis to decide whether it should be continued or not.
Although it is a natural criterion for beam search, in our case \emph{the unconstrained greedy hypothesis is often considerably more likely than the constrained hypothesis}, leading to repeated hypothesis replacement.
Note that this pattern aligns with previous work showing that many hypotheses in a beam search procedure often share the same prefix \citep{chan2025beamtrie}.
As a result, the algorithm introduced in Sec.~\ref{sec:cst_hyp} often reaches the end of the generation budget without producing a well-formed output.
To avoid this limitation, we propose a different criterion for hypothesis replacement that is only based on the score difference at the bifurcation position, that is, the position from which a candidate output switches from reasoning to constrained answering.

\textbf{Scoring via Bifurcation Penalty.}
Consider the greedy hypothesis $\vy = \langle \evy_1 \dots \evy_t\rangle$.
The bifurcation or deviation penalty at the next step is defined as:
\begin{align}\label{eq:penalty}
    \max_{v \in \overline{V}} p(v | \vx \odot \vy)~~-~~ \max_{v\in\overline{\operatorname{next}}_{\mathcal G}(\epsilon)} p(v | \vx \odot \vy),
\end{align}
In other words, it is the probability difference between the most probable token among the full vocabulary and the most probable token that can start a constrained answer.
If this penalty is lower than the one of the current constrained hypothesis, we replace the latter with a new hypothesis, setting $\vr = \vy$ and $\va = \langle \argmax_{v\in\overline{\operatorname{next}}_{\mathcal G}(\epsilon)} p(v | \vy) \rangle$, whose score is set to the penalty in Eq.~(\ref{eq:penalty}).
Note that the next greedy token may match a token in $\overline{\operatorname{next}}_{\mathcal G}(\epsilon)$,
meaning that the penalty in Eq.~(\ref{eq:penalty}) is equal to 0.
In this case, the penalty score of the constrained hypothesis is set to 0,
but it is updated at each step until the constrained hypothesis deviates from the greedy hypothesis.
If the constrained hypothesis generates a $\bot$ token while having a penalty score of 0, this means that greedy decoding results in a well-formed output.
In such a case, our algorithm returns this output.

\textbf{Bifurcation Penalty Choice.}
We consider two bifurcation penalty alternatives.
The difference of probability as in Equation~\ref{eq:penalty},
and the difference in logit space:\footnote{For computational stability, our implementation relies on a numerically stable operations \citep[Table~3]{lieisner2009first}.}
\begin{align}
    \max_{v \in \overline{V}} \log p(v | \vy)~~-~~ \max_{v\in\overline{\operatorname{next}}_{\mathcal G}(\epsilon)} \log p(v | \vy).
\end{align}

\textbf{Constrained Hypothesis Selection.}
We explore two possible hypothesis selection choices at the end of the algorithm: (1) returning the hypothesis of minimum penalty, (2) returning the last hypothesis.

The full algorithm is given in Alg.~\ref{alg:constr_derivation}.
It is important to note that, unlike the first constrained beam algorithm, the greedy state is allowed to end.
In this case, the last hypothesis still in generation is continued until it finishes.

\section{Related work}
\label{sec:related_work}

\textbf{Constrained Generation.}
Previous works have proposed methods to explicitly enforce LLM generations to be in the language generated by a grammar, for example, a regular or context-free language \citep[][\emph{inter alia}]{ghazvininejad2016generating, allauzen2014pushdown,shin-etal-2021-constrained,poesia2022synchromesh,geng2023grammarconstraineddecodingforstructured,pmlr-v235-beurer-kellner24a,ugare2024syncodellmgenerationgrammar,NEURIPS2023_cd40d0d6}.
Roughly speaking, all these methods constrain the next token prediction so that it leads to a prefix of the target language.
These methods have been implemented in several software libraries \citep{outlines, guidance, llamacpp}.
However, as per Propostion (\ref{prop:glutton_constrained}), they all fail on suffix-constrained generation tasks.
Other works have also proposed methods to introduce lexical constraints in free-form texts \citep{anderson2017guidedopenvocabulary,lu-etal-2021-neurologic,lu-etal-2022-neurologic,zhang2024adaptivelogicalcontrol}.
These constraints specify which words (or tokens, or concepts) should be included (somewhere) in the answer.
However, these methods are orthogonal to our suffix-constrained problem.

\textbf{Generation with Reasoning.}
\citet{tam2024speakfreely} show that LLMs' performance in reasoning tasks is heavily degraded when instructed to adhere to a structured output format, such as JSON, compared to free-form responses.
This degradation is further amplified when the model's output is explicitly constrained during generation using a grammar. Our suffix-constrained approach allows to benefit from reasoning steps while having constrained outputs for the end of the generation.
Differently, \citet{banerjee2025crane} propose to switch between free-form and grammar-constrained generation with special tokens that mark the start and end of constrained sub-sequences.
The model is encouraged with special prompt instructions to generate these special tokens.
This approach targets a different problem than ours and is fundamentally different to our suffix-constrained problem that do not rely on special prompts neither generating extra special transition tokens.
Some LLM families as Qwen 3 \citep{yang2025qwen3technicalreport} and DeepSeek \citep{Guo_2025} introduce explicit tokens \texttt{<think></think>} to isolate reasoning steps, while our solution can be applied to any model without further training.

\section{Experiments} \label{sec:experiments}

\begin{table*}[t!]
    \setlength{\tabcolsep}{4pt}
    \setlength{\aboverulesep}{0pt}
    \setlength{\belowrulesep}{0pt}
    \setlength{\cmidrulesep}{0pt}
    \setlength{\cmidrulekern}{0pt}
    \small
    \centering
    \begin{tabular}{l iji c ji c jij c ij}
        \toprule
        & \multicolumn{3}{c}{\textbf{Math}} && \multicolumn{2}{c}{\textbf{MCQ}}
        && \multicolumn{3}{c}{\textbf{Math}} && \multicolumn{2}{c}{\textbf{MCQ}}
        \\
        \cmidrule{2-4} \cmidrule{6-7} \cmidrule{9-11} \cmidrule{13-14}
        & \textbf{GMS} & \textbf{MATH} & \textbf{SVMP} && \textbf{ARC} & \textbf{CSQA}
        && \textbf{GMS} & \textbf{MATH} & \textbf{SVMP} && \textbf{ARC} & \textbf{CSQA}
        \\
        
        \midrule
        
        & \multicolumn{6}{c}{\textbf{OLMo 2 1B, Instruction-tuned}}
        && \multicolumn{6}{c}{\textbf{OLMo 2 13B, Instruction-tuned}}
        \\
        \cmidrule{2-7} \cmidrule{9-14}
        \textbf{Baselines} &&&&&&&&&&&&& \\
        $\hookrightarrow$ unconstrained                   & $63.9$ & $17.0$ & $68.0$ && $\underline{46.3}$ & $47.9$ && $84.9$ & $38.0$ & $\underline{83.6}$ && $83.7$ & $77.2$ \\
        $\hookrightarrow$ constrained only                & $01.5$ & $03.8$ & $22.3$ && $43.7$ & $\mathbf{51.3}$ && $16.1$ & $11.6$ & $52.6$ && $79.5$ & $\mathbf{78.0}$ \\
        
        \textbf{Naive methods} &&&&&&&&&&&&& \\
        $\hookrightarrow$ greedy pipeline          & $\mathbf{68.8}$ & $17.0$ & $69.3$ && $46.2$ & $47.9$ && $\mathbf{87.7}$ & $\underline{38.2}$ & $\mathbf{84.0}$ && $\underline{83.9}$ & $\underline{77.9}$ \\
        $\hookrightarrow$ constr.\ hyp.\ beam       & $66.5$ & $17.0$ & $65.3$ && $39.9$ & $39.9$ && $87.0$ & $35.0$ & $\mathbf{84.0}$ && $77.8$ & $72.6$ \\
        
        \textbf{Probability bif.} &&&&&&&&&&&&& \\
        $\hookrightarrow$ min penalty              & $18.7$ & $08.6$ & $25.3$ && $41.0$ & $48.8$ && $24.1$ & $14.0$ & $35.6$ && $75.5$ & $74.6$ \\
        $\hookrightarrow$ last hypothesis          & $65.1$ & $\mathbf{18.4}$ & $\mathbf{71.3}$ && $\mathbf{47.0}$ & $48.4$ && $\underline{87.1}$ & $\underline{38.2}$ & $\underline{83.6}$ && $\mathbf{84.0}$ & $77.8$ \\
        
        \textbf{Logits bif.} &&&&&&&&&&&&& \\
        $\hookrightarrow$ min penalty              & $42.6$ & $10.0$ & $54.0$ && $42.0$ & $\underline{49.9}$ && $38.0$ & $18.6$ & $52.0$ && $72.4$ & $73.3$ \\
        $\hookrightarrow$ last hypothesis          & $\underline{66.7}$ & $\underline{17.6}$ & $\underline{71.0}$ && $46.2$ & $48.0$ && $\mathbf{87.7}$ & $\mathbf{39.2}$ & $\mathbf{84.0}$ && $\underline{83.9}$ & $\mathbf{78.0}$ \\

        \midrule
        & \multicolumn{6}{c}{\textbf{Gemma 3 1B, Instruction-tuned}}
        && \multicolumn{6}{c}{\textbf{Gemma 3 12B, Instruction-tuned}}
        \\
        \cmidrule{2-7} \cmidrule{9-14}
        \textbf{Baselines} &&&&&&&&&&&&& \\
        $\hookrightarrow$ unconstrained                   & $43.1$ & $32.8$ & $\underline{55.3}$ && $\mathbf{52.3}$ & $\mathbf{49.9}$ && $82.1$ & $65.9$ & $89.0$ && $90.7$ & $\underline{78.1}$ \\
        $\hookrightarrow$ constrained only                & $00.0$ & $05.0$ & $01.6$ && $46.2$ & $42.7$ && $02.1$ & $19.2$ & $11.3$ && $88.3$ & $\mathbf{78.7}$ \\
        
        \textbf{Naive methods} &&&&&&&&&&&&& \\
        $\hookrightarrow$ greedy pipeline          & $\mathbf{44.4}$ & $32.8$ & $57.3$ && $\underline{51.2}$ & $\mathbf{49.9}$ && $\underline{91.6}$ & $65.8$ & $\mathbf{91.6}$ && $\underline{91.6}$ & $\mathbf{78.7}$ \\
        $\hookrightarrow$ constr.\ hyp.\ beam       & $\underline{44.1}$ & $26.0$ & $\mathbf{56.0}$ && $45.6$ & $48.6$ && $\mathbf{91.8}$ & $51.5$ & $\underline{91.0}$ && $86.1$ & $76.5$ \\
        
        \textbf{Probability bif.} &&&&&&&&&&&&& \\
        $\hookrightarrow$ min penalty              & $27.5$ & $16.2$ & $38.0$ && $36.8$ & $36.3$ && $61.7$ & $41.1$ & $68.6$ && $68.5$ & $63.2$ \\
        $\hookrightarrow$ last hypothesis          & $32.7$ & $\underline{35.6}$ & $47.0$ && $50.5$ & $\underline{49.2}$ && $74.6$ & $\mathbf{69.6}$ & $80.3$ && $\mathbf{91.7}$ & $\mathbf{78.7}$ \\
        
        \textbf{Logits bif.} &&&&&&&&&&&&& \\
        $\hookrightarrow$ min penalty              & $38.5$ & $17.2$ & $49.6$ && $35.6$ & $34.9$ && $71.4$ & $32.8$ & $68.0$ && $54.7$ & $40.5$ \\
        $\hookrightarrow$ last hypothesis          & $40.0$ & $\mathbf{36.4}$ & $52.6$ && $50.5$ & $48.8$ &&$78.5$ & $\underline{68.8}$ & $79.6$ && $\underline{91.6}$ & $\mathbf{78.7}$ \\
        
        \midrule
        & \multicolumn{6}{c}{\textbf{OLMo 2 13B, Pre-trained}}
        && \multicolumn{6}{c}{\textbf{Gemma 3 12B, Pre-trained}}
        \\
        \cmidrule{2-7} \cmidrule{9-14}
        \textbf{Baselines} &&&&&&&&&&&&& \\
        $\hookrightarrow$ unconstrained                   & $60.8$ & $12.2$ & $72.6$ && $38.3$ & $21.2$ && $15.0$ & $15.5$ & $20.6$ && $55.9$ & $45.6$ \\ 
        $\hookrightarrow$ constrained only                & $18.6$ & $09.6$ & $63.3$ && $\mathbf{72.7}$ & $\mathbf{66.3}$ && $14.3$ & $17.5$ & $50.6$ && $\mathbf{84.6}$ & $\mathbf{71.7}$ \\
        
        \textbf{Naive methods} &&&&&&&&&&&&& \\ 
        $\hookrightarrow$ greedy pipeline          & $\underline{69.1}$ & $12.2$ & $74.0$ && $42.3$ & $20.6$ && $14.2$ & $18.6$ & $19.6$ && $59.5$ & $46.6$ \\
        $\hookrightarrow$ constr.\ hyp.\ beam       & $67.2$ & $10.8$ & $\mathbf{79.6}$ && $34.9$ & $21.0$ && $24.2$ & $13.1$ & $55.0$ && $53.4$ & $43.8$ \\
        
        \textbf{Probability bif.} &&&&&&&&&&&&& \\
        $\hookrightarrow$ min penalty              & $39.8$ & $\mathbf{17.0}$ & $54.3$ && $42.1$ & $\underline{21.3}$ && $26.6$ & $22.2$ & $\underline{68.3}$ && $79.4$ & $64.3$ \\
        $\hookrightarrow$ last hypothesis          & $67.7$ & $\underline{15.2}$ & $74.6$ && $43.6$ & $21.0$ && $\mathbf{31.4}$ & $\mathbf{27.7}$ & $\underline{68.3}$ && $74.3$ & $61.1$ \\
        
        \textbf{Logits bif.} &&&&&&&&&&&&& \\ 
        $\hookrightarrow$ min penalty              & $16.6$ & $12.0$ & $31.6$ && $42.7$ & $20.3$ && $22.8$ & $18.1$ & $63.0$ && $\underline{80.8}$ & $\underline{66.6}$ \\
        $\hookrightarrow$ last hypothesis          & $\mathbf{69.9}$ & $14.0$ & $\underline{75.3}$ && $\underline{44.7}$ & $21.2$ && $\underline{31.1}$ & $\underline{27.0}$ & $\mathbf{68.6}$ && $73.4$ & $60.6$ \\

        \bottomrule
    \end{tabular}
    \caption{
        OLMo 2 1B, 13B IT \& PT and Gemma 3 1B, 12B IT \& PT models' accuracy with the different presented decoding algorithms.
        \emph{unconstrained} $\rightarrow$ Alg.\ \ref{alg:greedy_search}.
        \emph{unconstrained only} $\rightarrow$ Alg.\ \ref{alg:constrained_greedy_search}.
        \emph{greedy pipeline} $\rightarrow$ Alg.\ \ref{alg:pipeline_search}.
        \emph{constr. hyp.\ beam} $\rightarrow$ Alg.\ \ref{alg:constr_beam}.
        And \emph{bif.} (short for bifurcation) $\rightarrow$ Alg.\ \ref{alg:constr_derivation}.
    }
    \label{tab:accuracy_it_big_pt}
\end{table*}

Although these algorithms can be applied on any suffix-constrained language,
\textit{we choose to evaluate them in this paper on different question-answering tasks and answer formats, as the accuracy of the model on this task can be used as a proxy to asses how much each algorithm alter the models original distribution and performance.}
As these algorithms change only the answer format, \textit{we expect improvement in answer extractability and formatting}.
Thus, for evaluation, the LLM answer is extracted via standard pattern matching.

We experiment with OLMo 2 \citep{olmo20252olmo2furious} and Gemma 3 \citep{gemma3} to cover both open-source and closed-source model families for which we have access to different checkpoints and different model sizes.
We use 2 different checkpoints: pre-trained (PT) and instruction-tuned (IT).
We evaluate our method on three \textit{math reasoning} datasets: GSM8K \citep{cobbe2021trainingverifierssolvemath}, MATH500 \citep{lightman2023letsverifystepstep} and SVAMP \citep{patel2021svamp}, and two multiple-choice question (MCQ) \textit{commonsense reasoning} datasets: ARC-Challenge \citep{allenai2018arc} and CommonsenseQA \citep{talmor2019commonsenseqa}.
Practical considerations and evaluation settings are described in App.~\ref{app:practical_considerations}.

\begin{table*}[!ht]
    \setlength{\tabcolsep}{4pt}
    \setlength{\aboverulesep}{0pt}
    \setlength{\belowrulesep}{0pt}
    \setlength{\cmidrulesep}{0pt}
    \setlength{\cmidrulekern}{0pt}
    \small
    \centering
    \begin{tabular}{l iji c ji c jij c ij}
        \toprule
        & \multicolumn{3}{c}{\textbf{Math}} && \multicolumn{2}{c}{\textbf{MCQ}}
        && \multicolumn{3}{c}{\textbf{Math}} && \multicolumn{2}{c}{\textbf{MCQ}}
        \\
        \cmidrule{2-4} \cmidrule{6-7} \cmidrule{9-11} \cmidrule{13-14}
        & \textbf{GMS} & \textbf{MATH} & \textbf{SVMP} && \textbf{ARC} & \textbf{CSQA}
        && \textbf{GMS} & \textbf{MATH} & \textbf{SVMP} && \textbf{ARC} & \textbf{CSQA}
        \\
        \midrule
        & \multicolumn{6}{c}{\textbf{OLMo 2 13B, Pre-trained}}
        && \multicolumn{6}{c}{\textbf{OLMo 2 13B, Instruction-tuned}}
        \\
        \cmidrule{2-7} \cmidrule{9-14}
        $\hookrightarrow$ unconstrained          & $84$ & $30$ & $92$ && $88$ & $97$ && $100$ & $90$ & $100$ && $100$ & $100$ \\
        $\hookrightarrow$ constr.\ hyp.\ beam          & $91$ & $1$  & $97$ && $0$  & $0$ && $100$ & $2$  & $100$ && $33$  & $26$ \\
        $\hookrightarrow$ Bif. penalty           & \multicolumn{6}{c}{\rule[2pt]{1cm}{0.5pt}$~\mathbf{100}~$\rule[2pt]{1cm}{0.5pt}} && \multicolumn{6}{c}{\rule[2pt]{1cm}{0.5pt}$~\mathbf{100}~$\rule[2pt]{1cm}{0.5pt}} \\
        
        \midrule
        & \multicolumn{6}{c}{\textbf{Gemma 3 12B, Pre-trained}}
        && \multicolumn{6}{c}{\textbf{Gemma 3 12B, Instruction-tuned}}
        \\
        \cmidrule{2-7} \cmidrule{9-14}
        $\hookrightarrow$ unconstrained          & $8$ & $41$ & $5$ && $33$ & $17$     && $100$ & $80$ & $99$ && $100$ & $100$ \\
        $\hookrightarrow$ constr.\ hyp.\ beam          & $45$ & $0$  & $59$ && $0$  & $0$    && $90$ & $0$  & $91$ && $4$  & $6$  \\
        $\hookrightarrow$ Bif.\ penalty          & $94$ & $99$ & $99$ && $100$ & $100$ && \multicolumn{6}{c}{\rule[2pt]{1cm}{0.5pt}$~\mathbf{100}~$\rule[2pt]{1cm}{0.5pt}} \\
        \bottomrule
    \end{tabular}
    \caption{Proportion of outputs that end with $\bot$ before hitting the budget limit (finished gen.). \emph{Bif.\ penalty} refers to all bifurcation penalty variants.}
    \label{tab:prop_finished_large}
\end{table*}

\subsection{Results and Discussion} \label{sec:results_discussion}
Results for IT and PT models are presented in Table~\ref{tab:accuracy_it_big_pt}.\footnote{General accuracy of 1B PT models is very low and thus is left in App.~\ref{app:accuracy_small_pt}.}
In addition to the presented methods, we include two baselines:
unconstrained decoding, and constrained generation that skips the reasoning part and starts from the constrained answer.
Using pattern matching, the wrong answer can be extracted from a response in the following scenarios:
(1) the response is not finished because of a generation budget overflow;
(2) the response contains a final answer, but it is formatted incorrectly;
(3) the response contains a final answer with correct formatting, but it is followed with another substring that matches the formatting; the final substring matching the formatting is extracted instead of the actual final answer.
Examples of these behaviors are shown in App.~\ref{app:motivation}.
Based on these failure cases, we examine the strengths and weaknesses of each approach in the following.

Overall, on OLMo 2, greedy pipeline as well as last hypothesis selection via bifurcation penalty (in prob.\ and log-prob.\ spaces) consistently improves over unconstrained greedy generation across all tasks.
On Gemma 3, we observe a similar pattern except for GSM8K and SVAMP that expect a numerical final answer.
On this model, greedy pipelines offers the most stable improvement over unconstrained generation across tasks.

\textbf{Bifurcation penalty vs. unconstrained search.}
\emph{The bifurcation penalty with the last hypothesis selection solves all the previously mentioned failure cases from which suffers unconstrained generation}.
This results in improved accuracy with bifurcation penalty with last hypothesis selection compared to unconstrained generation for almost all tasks and models.

\textbf{Bifurcation penalty with last hypothesis selection v.s.\ greedy pipeline.}
These methods tend to produce very similar results, as the last hypothesis often starts from the last position.
However, the bifurcation penalty last hypothesis selection algorithm has more freedom in selecting the start of the constrained part.
Moreover, this algorithm produces a valid answer within the limited generation tokens budget,
unlike the greedy pipeline that can overflow the budget without generating an answer (\emph{i.e.}, when the first unconstrained phase overflows the budget, 1st failure case).

\textbf{Minimum penalty selection.}
Selecting the minimum penalty hypothesis, on the other hand, tends to produce shorter outputs (see App.~\ref{app:response_length}), which in mathematical tasks skips derivation steps and produces intermediate results instead of the final expected answer, \emph{e.g.}, \texttt{\textbackslash boxed\{\textbackslash sqrt\{117\}\}} instead of \texttt{\textbackslash boxed\{3\textbackslash sqrt\{13\}\}}.
In most of the studied cases, this algorithm performs poorly compared to the unconstrained baseline.
The exception is the results of PT models on MCQ tasks, which require knowledge invocation rather than long reasoning.

\textbf{Constrained only.}
Interestingly, omitting the reasoning part of the output yields different performance patterns across tasks.
For instance, MCQ tasks suffer almost no performance degradation.
On the contrary, the constrained generation method achieves the best performance in some cases, as with OLMo 2 13B PT on ARC (resp.\ CSQA): 72.7 (resp.\ 66.3), compared to 44.7 (resp.\ 21.3) for the second-best score, and similarly for Gemma 3 12B PT.
Whereas for math tasks that usually require more derivation and reasoning steps, we observe the opposite pattern: the performance drop compared to unconstrained decoding is considerable, \emph{e.g.}, from 84.9 to 16.1, from 63.9 to 1.5, etc.

\subsection{Further Analysis} \label{sec:further_analysis}

The constrained hypothesis beam generation algorithm produces longer outputs as the hypothesis beam is frequently interrupted and less likely to terminate (first failure case).
Unlike the minimum bifurcation penalty algorithm, which has large performance drops in some settings, constrained beam performance remains steady and shows no large drops, though it still underperforms the unconstrained baseline in most settings.
Table~\ref{tab:prop_finished_large} shows the proportion of terminated outputs for each algorithm before the budget limit is reached.
We can notice that \emph{all bifurcation algorithm variants almost always successfully finish their generation}.
Unconstrained decoding finishes its generation more often for IT models than for PT models, as the latter models are trained on raw text with no specific $\bot$ delimitation.
Finally, the constrained hypothesis beam algorithm fails to complete its generation in a considerable portion of cases on MATH500, ARC, and CSQA.
This can explain the performance drop of the constrained hypothesis beam algorithm compared to unconstrained decoding on these tasks.
Since the only termination path of the constrained hypothesis beam algorithm is via the constrained hypothesis, as the unconstrained hypothesis cannot generate the $\bot$ termination symbol, we can conclude in this case that the constrained hypothesis is continually interrupted.
Further investigation of this interruption phenomenon is given in App.~\ref{app:response_length}.

\section{Conclusion} \label{sec:conclusion}

In this work,
we introduce the suffix-constrained generation problem and propose several greedy search algorithms for causal LLM.
Using question-answering problems as an evaluation proxy,
we show that our algorithms allow to constrain outputs without negatively impacting downstream task performances.

Our work exposes the structural and algorithmic challenges underlying this setting and opens a new research direction in constrained generation.

\section*{Limitations} \label{sec:limitations}

An open question is how to adapt ancestral sampling for suffix-constrained generation,
as grammar constraints distort the output distribution \citep{li2025adaptrackconstraineddecodingdistorting, park2024grammaraligneddecoding}.
Such research questions are left for future work.

\section*{Acknowledgments}

We thank Aurélie Névéol and Karën Fort for fruitful discussions that led to this research work.

This work is supported by the SEMIAMOR (CE23-2023-0005) and InExtenso (ANR-23-IAS1-0004) project grants given by the French National Research Agency (ANR). This work was granted access to the HPC resources of IDRIS under the allocation 2024-AD011015801 made by GENCI.


\bibliography{custom}

\appendix

\section{Suffix-constrained Generation for Question Answering}
\label{app:motivation}

In this section, we give motivation example for the suffix-constrained generation problem.

Table \ref{tab:failure_examples} shows example of generated responses using greedy unconstrained decoding and our method, on different tasks and models.
These examples illustrate different failure cases (presented in Section \ref{sec:results_discussion}) that suffix-constrained generation solves.

The first failure case, that is, the case of generation budget overflow, is illustrated in the generations of the pre-trained model in Table \ref{tab:failure_examples} for both math and multiple-choice question tasks.
We can see in the math reasoning example that an intermediate number is extracted as the final answer, as the model did not finish its reasoning.
And similarly for the question answering task, as the unconstrained generation did not produce any final answer before overflowing the budget.
The example response using suffix-constrained generation algorithms avoids this issue by producing a natural interruption of the reasoning to generate the final answer in the expect format, and before the exhaustion of the generation budget.

The instruction-tuned model unconstrained generation in Table \ref{tab:failure_examples} shows an example of the third failure case, that is when the model generates the final answer in the expected format, but follows it with another substring in the same format.
In this case, deterministic pattern-matching extraction methods would extract the final substring instead of the actual final answer intended by the model, which leads to misevaluation of the model's response.
For instance, the math reasoning example shows that, instead of the actual answer ``576'', the tailing ``96'' is extracted, which corresponds to a contextual detail.
The suffix-constrained generation on the other hand follows the unconstrained response with an explicit phrasing of the final answer in the correct format, thus guaranteeing that the correct answer is extracted.

\begin{table*}[!ht]
    \centering
    \scriptsize
    \begin{tabular}{lp{0.85\linewidth}}
        \toprule
        \multicolumn{2}{c}{\textbf{\small Math reasoning on GSM8K}} \\
        \midrule
        \multicolumn{2}{l}{\textbf{13B, instruction-tuned}}
        \\
        Input & If 6 potatoes makes 36 hash browns, how many hash browns can you make out of 96 potatoes?
        \\
        Greedy search & [\dots] Therefore, you can make 576 hash browns out of \textcolor{Red}{96} potatoes.
        \\
        Our method & [\dots] Therefore, you can make 576 hash browns out of 96 potatoes. The answer is: \textcolor{Red}{576}.
        \\[0.5em]
        
        \multicolumn{2}{l}{\textbf{13B, pre-trained}}
        \\
        Input & James has 6 more candies than Robert. John has twice as many candies as Robert. If John has 54 candies, how many more candies does John have than James?
        \\
        Greedy search & [\dots]  Since John has 54 candies and James has \textcolor{Red}{33} candies, we can find the difference by subtracting James's<end of budget>
        \\
        Our method & [\dots] Since John has 54 candies and The answer is: \textcolor{Red}{21}.
        \\
        \midrule
        \multicolumn{2}{c}{\textbf{\small Question answering on ARC}} \\
        \midrule
        \multicolumn{2}{l}{\textbf{13B, instruction-tuned}}
        \\
        Input & Choose the correct answer to the following multiple-choice question. Question: A euglena cell has a structure called an eyespot that detects light. A paramecium does not have an eyespot, and so it cannot detect light. Why doesn't a paramecium need an eyespot?
        
        A). A paramecium does not live where there is light.~~~~~~~~B). A paramecium does not photosynthesize its own food.
        
        C). A paramecium lives in shallower water than a euglena~~~~~~~~D). A paramecium uses cilia instead of a flagellum to move.
        
        Provide your reasoning about the answer and finish your answer with the letter corresponding to the correct option (e.g., A, B, C, or D).
        \\
        Greedy search & [\dots] Therefore, the correct answer is: B. \textcolor{Red}{A} paramecium does not photosynthesize its own food.
        \\
        Our method & [\dots] Therefore, the correct answer is: B. A paramecium does not photosynthesize its own food. The answer is: \textcolor{Red}{B}
        \\[0.5em]
        
        \multicolumn{2}{l}{\textbf{13B, pre-trained}}
        \\
        Input & Choose the correct answer to the following multiple-choice question.

        Question: Approximately 71 \% of the Earth's surface is covered with water. The majority of this is salt water and is found in the oceans. Where is the majority of the Earth's fresh water found?

        A). ice.~~~~~~~~B). lakes.~~~~~~~~C). rivers.~~~~~~~~D). the atmosphere
        
        Provide your reasoning about the answer and finish your answer with the letter corresponding to the correct option (e.g., A, B, C, or D). 
        \\
        Greedy search & [\dots] The majority of the Earth's fresh water is found in ice. The majority of the Earth's fresh water is found<end of budget> \textcolor{Red}{(no answer is extracted)}
        \\
        Our method & [\dots] The majority of the Earth's fresh water is found in ice. The majority of the The answer is: \textcolor{Red}{A}
        \\
        \bottomrule
    \end{tabular}

    \caption{Examples of model outputs generated by OLMo2 (13B, instruction-tuned for the first example of each task and pre-trained for the second example of each task) using greedy unconstrained search and one of the suffix-constrained methods we introduce (bifurcation penalty with last hypothesis selection, see Section~\ref{sec:bifurcation_penalty}).
    For math reasoning on GSM8K, the evaluation protocol extracts the last number.
    For question answering on ARC, the evaluation protocol extracts the last individual (capital) letter in the set $\{A, B, \dots, J\}$.
    Extracted answers from each output are colored in red.
    We can see that the extracted answer from the greedy search response is not the intended final answer. Instead, another substring with the same formatting is extracted as the answer.
    Pre-trained only models tend not to finish their generation and enter a looping response until hitting the end of the allocated generation budget.
    Our proposed method finds a correct continuation that interrupts the loop and produces the final answer.
    }
    \label{tab:failure_examples}
\end{table*}

\section{Practical Considerations.} \label{app:practical_considerations}
Each prompt is augmented with a final instruction specifying the expected formatting of the final answer.
After the output is generated, the answer is extracted using a simple matching algorithm.
For evaluation, we rely on exact match against the gold answer.
If no match is found, no answer is extracted, and the LLM's response is evaluated as false.

For constrained generation, the formatting is mirrored by a formal grammar.
In our experiment, we mainly rely on regular languages, for example:
\begin{quote}
    \texttt{The answer is: [+-]?\textbackslash d+(\textbackslash .\textbackslash d+)?} \qquad
    \texttt{The answer is: \textbackslash b[A-J]\textbackslash b} 
\end{quote}
However, for Math500 the answer follows \LaTeX\ mathematical notation.
As such, we need to constrain the answer to close all opened brackets, meaning that the language is context-free, in a nutshell:
\begin{quote}
    \texttt{The answer is: \textbackslash boxed\{ ... \}}
\end{quote}
where ``\dots'' is replaced with the context-free well bracketed \LaTeX\ expression constraints.

Note that our grammars account for all possible tokenizations of sequences.

\section{Answer Patterns} \label{app:answer_patterns}

For the MATH500 task for which answers are expected to be surrounded by a \texttt{\textbackslash boxed\{\}} in \LaTeX\ format, the following final instruction is added to each prompt:
\texttt{Present the answer in LaTex format: \textbackslash boxed\{Your answer\}}.

For other mathematical reasoning tasks (GSM8K and SVAMP), answers are extracted as the last, optionally signed, decimal or natural number in the response.

Multiple-choice question tasks, on the other hand, expect a final option letter.
The following final instruction is added to each prompt: \texttt{Provide your reasoning about the answer and finish your answer with the letter corresponding to the correct option (e.g., A, B, C, or D).}

\section{Constrained Hypothesis Beam Search}
\label{app:constrained_hypothesis_beam}

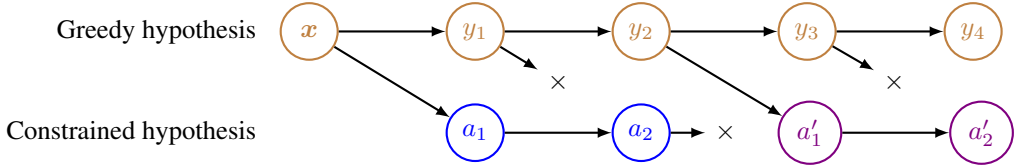
\begin{figure*}
\centering
\small
\begin{tikzpicture}[
    node distance=1.6cm,
    every path/.style={->, thick}
]
\tikzset{
  beamnode/.style={circle, draw, minimum size=7.5mm, thick},
}
\newlength{\hdist}
\setlength{\hdist}{40pt}
\newlength{\hdistp}
\setlength{\hdistp}{40pt}

\node[beamnode, brown] (g1) at (0,0) {$\vx$};
\node[beamnode, brown] (g2) [right=\hdist of g1] {$\evy_1$};
\node[beamnode, brown] (g3) [right=\hdist of g2] {$\evy_2$};
\node[beamnode, brown] (g4) [right=\hdistp of g3] {$\evy_3$};
\node[beamnode, brown] (g5) [right=\hdist of g4] {$\evy_4$};

\draw (g1) -- (g2);
\draw (g2) -- (g3);
\draw (g3) -- (g4);
\draw (g4) -- (g5);

\node[left=5pt of g1,anchor=east] (G) {Greedy hypothesis};

\node[below=16pt of G.south east,anchor=east] (H) {Constrained hypothesis};

\node[beamnode, blue] (h2) at ($(H.east -| g2.south)$) {$\eva_1$};
\node[beamnode, blue] (h3) [right=\hdist of h2] {$\eva_2$};
\node[beamnode, violet] (h4) [right=\hdistp of h3] {$\eva'_1$};
\node[beamnode, violet] (h5) [right=\hdist of h4] {$\eva'_2$};


\node (h3end) at ($(h3)!0.5!(h4)$) {$\times$};
\node (g2end) at ($(g2)!0.5!(h3)$) {$\times$};
\node (g4end) at ($(g4)!0.5!(h5)$) {$\times$};
\draw (g2) -- (g2end);
\draw (g4) -- (g4end);

\draw (g1) -- (h2);
\draw (h2) -- (h3);
\draw (h3) -- (h3end);
\draw (g3) -- (h4);
\draw (h4) -- (h5);

\end{tikzpicture}
\caption{Constrained hypothesis beam search with a greedy hypothesis (top) and a constrained hypothesis (bottom).
At each generation step, the constrained hypothesis can be replaced (or not) with a new one.
For example, after step 1, we choose to continue with the current constrained hypothesis (token $\eva_2$ is added after $\eva_1$).
However, after step two, the constrained hypothesis has been replaced with a new one, whose first token is denoted $\eva'_1$.
At step 4,
the greedy hypothesis is defined as $\vy = \langle \evy_1, \evy_2, \evy_3, \evy_4 \rangle$,
whereas the constrained one is $\vr \odot \langle \eva'_1, \eva'_2\rangle$,
where the reasoning part is composed of tokens from the greedy hypothesis before the replacement, that is $\vr = \langle \evy_1, \evy_2\rangle$.}
\label{fig:constrained_beam}
\end{figure*}

Figure~\ref{fig:constrained_beam} shows an illustration of the constrained hypothesis beam search with 2 hypothesis: the greedy unconstrained hypothesis and the constrained hypothesis.
The figure shows the interruption process of the constrained hypothesis beam by a new constrained hypothesis starting from the current unconstrained beam.

\section{Response Length}
\label{app:response_length}

\begin{figure}
    \centering
    \begin{tikzpicture}
        \begin{axis}[
            ylabel={Entropy difference},
            xlabel={Generation step},
            height=4cm,
            width=\linewidth,
            ymax=1.1,
            enlarge x limits = 0,
            tick label style={
                font=\tiny
            },
            label style={
                font=\tiny
            },
            legend style={
                font=\tiny
            },
            legend columns=2,
            legend style = {
                at = {(0.5, 0.98)},
                anchor = north,
            },
            extra y ticks={0},
            extra y tick style={
                grid style={
                    dotted,
                    color=black,
                }, 
                grid=major,
            },
        ]
            \addplot [mark=none, thick, color=blue] table[x=x,y=y] {figures/data/min_entropies_ps_gsm8k.dat};
            \addlegendentry{GMS8K}
            
            \addplot [mark=none, thick, color=red] table[x=x,y=y] {figures/data/min_entropies_ps_math500.dat};
            \addlegendentry{MATH500}
            
            %
        \end{axis}
    \end{tikzpicture}
    \caption{The average min-entropy difference between the next-token distribution under the unconstrained hyp.\ and the second token in a new-constraint hyp., at each generation step, on OLMo 2 13B IT.
    }
    \label{fig:min_entropy_diff}
\end{figure}
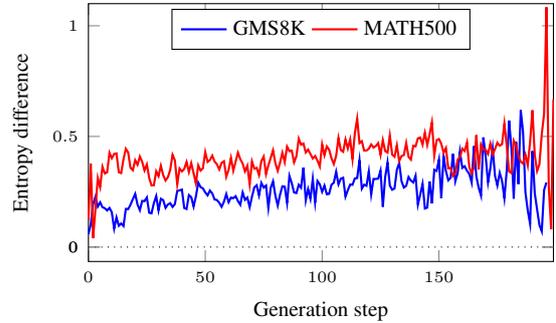

\begin{figure*}[t]
    \centering
    \begin{tikzpicture}[
    greedy/.style={draw=black, pattern={horizontal lines}, pattern color=red},
    pipeline/.style={draw=black, pattern={vertical lines}, pattern color=black},
    constr/.style={draw=black, pattern={grid}, pattern color=red},
    minprob/.style={draw=black, pattern={crosshatch}, pattern color=black},
    lastprob/.style={draw=black, pattern={dots}, pattern color=red},
    minlogit/.style={draw=black, pattern={north east lines}, pattern color=black},
    lastlogit/.style={draw=black, pattern={north west lines}, pattern color=red},
]
\begin{groupplot}[
    ybar = 0,
    group style = {
        group size=1 by 2,
        vertical sep=0.4cm,
    },
    width = 14cm,
    height = 3cm,
    enlarge x limits = {abs=1.7cm},
    symbolic x coords = {
        1bpt,
        1bit,
        13bpt,
        13bit,
    },
    ymin = 0,
    ymax = 1056,
    xticklabel = \empty,
    tick label style={
        font=\tiny
    },
    legend style = {
        at={(0.47,-0.1)},
        anchor=north,
        legend columns=4,
        font=\scriptsize,
    },
    y label style = {
        font=\scriptsize,
    }
]
\nextgroupplot[
    ylabel = {Math tasks},
]
\begin{scope}[on background layer]
\foreach \i in {0,2}{
    \pgfmathsetmacro{\currL}{\i/4}
    \pgfmathsetmacro{\currR}{(\i+1)/4}
    
    \edef\temp{
        \noexpand\fill[color=gray!7, draw=none]
            (rel axis cs:\currL, -1.45) rectangle (rel axis cs:\currR, 1.3);
    }
    \temp
}

\node[align=center] at (rel axis cs:0/4 + 1/8, 1.12) {\bfseries \scriptsize OLMo 2 1B-PT};
\node[align=center] at (rel axis cs:1/4 + 1/8, 1.12) {\bfseries \scriptsize OLMo 2 1B-IT};
\node[align=center] at (rel axis cs:2/4 + 1/8, 1.12) {\bfseries \scriptsize OLMo 2 13B-PT};
\node[align=center] at (rel axis cs:3/4 + 1/8, 1.12) {\bfseries \scriptsize OLMo 2 13B-IT};
\end{scope}

\addplot[greedy] coordinates {
({1bpt}, 410.495044832468)
({1bit}, 330.015101462954)
({13bpt}, 469.071260028315)
({13bit}, 237.418593676262)
};

\addplot[pipeline] coordinates {
({1bpt}, 415.290231241152)
({1bit}, 336.818310523832)
({13bpt}, 474.301557338367)
({13bit}, 244.970268994809)
};

\addplot[constr] coordinates {
({1bpt}, 379.958942897593)
({1bit}, 411.184521000472)
({13bpt}, 437.892873997168)
({13bit}, 359.56158565361)
};

\addplot[minprob] coordinates {
({1bpt}, 36.5563945257197)
({1bit}, 178.546956111373)
({13bpt}, 133.815950920245)
({13bit}, 108.929211892402)
};

\addplot[lastprob] coordinates {
({1bpt}, 401.287871637565)
({1bit}, 328.467201510146)
({13bpt}, 466.062765455404)
({13bit}, 242.196319018405)
};

\addplot[minlogit] coordinates {
({1bpt}, 50.0089664936291)
({1bit}, 191.436054742803)
({13bpt}, 123.956111373289)
({13bit}, 127.35913166588)
};

\addplot[lastlogit] coordinates {
({1bpt}, 402.472392638037)
({1bit}, 328.177442189712)
({13bpt}, 466.272770174611)
({13bit}, 243.842850401133)
};

\nextgroupplot[
    ylabel = {MCQ tasks},
]
\addplot[greedy] coordinates {
({1bpt}, 281.068951107397)
({1bit}, 186.292519849561)
({13bpt}, 317.890931884664)
({13bit}, 183.855411617217)
};

\addplot[pipeline] coordinates {
({1bpt}, 284.781863769327)
({1bit}, 191.282072712077)
({13bpt}, 322.504387797743)
({13bit}, 188.849143334726)
};

\addplot[constr] coordinates {
({1bpt}, 1022.43836188884)
({1bit}, 1021.27162557459)
({13bpt}, 1023.97659841204)
({13bit}, 794.309653155036)
};

\addplot[minprob] coordinates {
({1bpt}, 20.4412870873381)
({1bit}, 89.9139155871291)
({13bpt}, 45.9180944421229)
({13bit}, 95.2264939406603)
};

\addplot[lastprob] coordinates {
({1bpt}, 266.08023401588)
({1bit}, 189.745089845382)
({13bpt}, 314.386126201421)
({13bit}, 186.965315503552)
};

\addplot[minlogit] coordinates {
({1bpt}, 28.888006686168)
({1bit}, 83.3372335979941)
({13bpt}, 46.8800668616799)
({13bit}, 95.3506059339741)
};

\addplot[lastlogit] coordinates {
({1bpt}, 264.507312996239)
({1bit}, 191.137902214793)
({13bpt}, 315.64646886753)
({13bit}, 188.062682824906)
};

\legend{{Greedy gen.}, {Pipeline gen.}, {Constr. beam}, {Bif. probs. min}, {Bif. probs. last}, {Bif. logits min}, {Bif. logits last}}
\end{groupplot}

\end{tikzpicture}
    \caption{Average response lengths, per model, per group of tasks, and per generation method. ``pt'' stands for pre-trained and ``it'' stands for instruction-tuned. The maximum number of allowed answer tokens is 1024.}
    \label{fig:avg_response_length}
\end{figure*}

Figure \ref{fig:avg_response_length} shows the average response lengths of different generation methods on different models and tasks.
As discussed in Section \ref{sec:further_analysis} and Table \ref{tab:prop_finished_large}, Constrained hypothesis beam overflows the generation budget on MCQ tasks, while Bifurcation Penalty with minimum penalty selection produces on average shorter responses.
Other methods produce responses of similar length to unconstrained generation.

To investigate the interruption phenomenon discussed in Section~\ref{sec:further_analysis}, we study in Figure \ref{fig:min_entropy_diff} the entropy difference of the model distribution on the constrained hypothesis and on the unconstrained hypothesis after a single constrained step. More exactly, we are interested in this entropy difference:
$H \big [~ p(\cdot | \vx \odot \vy \odot \langle \argmax_{v \in \overline{\operatorname{next}}_\mathcal G (\epsilon)} p(v | \vx \odot \vy) \rangle ~\big ]
-
H \big [~ p(\cdot | \vx \odot \vy \odot \langle \argmax_{v \in V} p(v | \vx \odot \vy) \rangle ~\big ]$,
where $H$ is the min-entropy \cite{renner2004minent} .
We study in this figure the case of GSM8K and MATH500 on OLMo 2 13B IT, on which this model exhibits very contrasting constrained hypothesis interruption behaviors as per Table \ref{tab:prop_finished_large}.
We observe that, on MATH500, this entropy difference is on average higher than that on GSM8K at almost all generation positions.
A large difference in entropy indicates a drop in confidence from taking a constrained step, leading to a larger difference in scores between the greedy and the constrained beam.
This difference favors the interruption of the current constrained hypothesis by a newer constrained hypothesis starting from an advanced decoding position.

\section{Small Pre-trained Models Results} \label{app:accuracy_small_pt}

\begin{table*}[!ht]
    \small
    \centering
    \begin{tabular}{l iji c ji c jij c ij}
        \toprule
        & \multicolumn{3}{c}{\textbf{Math}} && \multicolumn{2}{c}{\textbf{MCQ}}
        && \multicolumn{3}{c}{\textbf{Math}} && \multicolumn{2}{c}{\textbf{MCQ}}
        \\
        \cmidrule{2-4} \cmidrule{6-7} \cmidrule{9-11} \cmidrule{13-14}
        & \textbf{GMS} & \textbf{MATH} & \textbf{SVMP} && \textbf{ARC} & \textbf{CSQA}
        && \textbf{GMS} & \textbf{MATH} & \textbf{SVMP} && \textbf{ARC} & \textbf{CSQA}
        \\

        \midrule
        & \multicolumn{6}{c}{\textbf{OLMo 2 1B, Pre-trained}}
        && \multicolumn{6}{c}{\textbf{Gemma 3 1B, Pre-trained}}
        \\
        \cmidrule{2-7} \cmidrule{9-14}
        \multicolumn{4}{l}{\textbf{Baselines}} &&&& \\
        unconstrained                   & $17.5$ & $00.8$ & $20.6$ && $18.0$ & $16.2$ && $\underline{01.4}$ & $00.0$ & $\mathbf{01.6}$ && $22.6$ & $17.7$ \\ 
        constrained only                & $04.4$ & $00.4$ & $\mathbf{33.6}$ && $25.6$ & $20.8$ && $01.2$ & $\mathbf{02.8}$ & $01.0$ && $\underline{26.0}$ & $\underline{19.8}$ \\
        
        \multicolumn{4}{l}{\textbf{Naive methods}} &&&& \\ 
        greedy pipeline          & $\underline{18.1}$ & $01.4$ & $22.0$ && $24.4$ & $25.0$ && $00.8$ & $\underline{01.0}$ & $01.0$ && $21.5$ & $18.0$ \\
        constr.\ hyp.\ beam       & $12.8$ & $01.2$ & $20.3$ && $20.0$ & $22.5$ && $\mathbf{01.5}$ & $00.0$ & $\underline{01.3}$ && $22.0$ & $\mathbf{20.2}$ \\
        
        \multicolumn{4}{l}{\textbf{Probability bif.}} &&&& \\
        min penalty              & $05.0$ & $00.4$ & $\underline{26.0}$ && $26.4$ & $22.8$ && $01.2$ & $\underline{01.0}$ & $00.6$ && $22.8$ & $19.0$ \\
        last hypothesis          & $\mathbf{19.4}$ & $\underline{02.2}$ & $24.6$ && $\underline{29.9}$ & $\underline{27.6}$ && $00.9$ & $00.8$ & $00.6$ && $\mathbf{26.2}$ & $18.5$ \\
        
        \multicolumn{4}{l}{\textbf{Logits bif.}} &&&& \\ 
        min penalty              & $07.8$ & $02.0$ & $17.3$ && $29.0$ & $26.9$ && $01.3$ & $00.2$ & $01.0$ && $22.8$ & $19.0$ \\
        last hypothesis          & $\mathbf{19.4}$ & $\mathbf{02.4}$ & $23.6$ && $\mathbf{30.2}$ & $\mathbf{28.1}$ && $01.3$ & $00.4$ & $01.0$ && $24.5$ & $19.6$ \\

        \bottomrule
    \end{tabular}
    
    \caption{OLMo 2 1B PT and Gemma 3 1B PT models' accuracy with the different presented decoding algorithms.}
    \label{tab:accuracy_small_pt}
\end{table*}

Results of small PT models are shown in Table \ref{tab:accuracy_small_pt}.
Results of OLMo 2 1B PT are consistent with those of larger models in Table \ref{tab:accuracy_it_big_pt};
Bifurcation penalty with last hypothesis selection outperforms other methods and consistently improve over unconstrained generation.
Scores for Gemma 3 1B PT are close to 0 on Math tasks, given us no margin for interpretation.

\section{Computation Resources} \label{app:computation_resources}

All evaluation experiments were carried on NVIDIA A100 80GB GPUs and parallelized over 8 GPUs for efficiency.
Evaluations span approximately from 1 hour to 5 hours depending on the model size and dataset size.

\end{document}